\documentclass{article}

\usepackage{microtype}
\usepackage{subfigure}
\usepackage{booktabs}

\usepackage[nonatbib,final]{bdl_2018}

\usepackage{amsmath, amsthm}
\usepackage[linesnumbered,algoruled,boxed,lined]{algorithm2e}
\usepackage{float}
\usepackage{hyperref}
\usepackage{graphicx}
\usepackage{amsfonts}
\usepackage{cite}

\usepackage{multirow}
\usepackage{makecell}
\usepackage{bbm}

\DeclareMathOperator*{\argmin}{arg\,min}

\newcommand{\norm}[1]{\left\lVert#1\right\rVert}

\usepackage{stackengine}

\makeatletter
\newcommand{\distas}[1]{\mathbin{\overset{#1}{\kern\z@\sim}}}%

\newtheorem{theorem}{Theorem}
\newtheorem{definition}{Definition}
\newtheorem{lemma}{Lemma}

\newtheorem{cor}{Corollary}

\newcommand{\bas}[1]{\begin{align*}#1\end{align*}}
\newcommand{\ba}[1]{\begin{align}#1\end{align}}

\newcommand{\bE}{\mathbb{E}}
\newcommand{\bH}{\mathbb{H}}
\newcommand{\bR}{\mathbb{R}}

\newcommand{\cN}{\mathcal{N}}

\newcommand{\cZ}{\mathcal{Z}}
\newcommand{\cX}{\mathcal{X}}
\newcommand{\cI}{\mathcal{I}}
\newcommand{\po}{\text{Pois}}
\newcommand{\ga}{\text{Gamma}}

\newcommand{\kl}{\text{KL}}

\newcommand{\beqs}{\vspace{0mm}\begin{eqnarray}}
\newcommand{\eeqs}{\vspace{0mm}\end{eqnarray}}
\newcommand{\barr}{\begin{array}}
\newcommand{\earr}{\end{array}}

\newcommand{\xv}{\boldsymbol{x}}

\newcommand{\zv}{\boldsymbol{z}}

\newcommand{\gammav}[0]{{\boldsymbol{\gamma}} }

\newcommand{\thetav}{\boldsymbol{\theta}}

\newcommand{\muv}[0]{{\boldsymbol{\mu}}}

\newcommand{\phiv}{\boldsymbol{\phi}}

\newcommand{\E}{\mathbb{E}}

\newcommand{\given}{\,|\,}

\usepackage{booktabs}
\usepackage{siunitx}

\usepackage{color,soul}

\definecolor{alizarin}{rgb}{0.82, 0.1, 0.26}

\begin{document}
\title{
Semi-Implicit Generative Model
}
\author{
  Mingzhang Yin \\
  Department of Statistics and Date Science\\
  The University of Texas at Austin\\
  Austin, TX 78712 \\
  \texttt{mzyin@utexas.edu} \\  
   \And
Mingyuan Zhou \\
  McCombs School of Business \\
   The University of Texas at Austin\\
     Austin, TX 78712 \\
   \texttt{mingyuan.zhou@mccombs.utexas.edu} \\
   }
\maketitle

\begin{abstract} 
To combine explicit and implicit generative models, we introduce semi-implicit generator (SIG) as a flexible hierarchical model that can be trained in the maximum likelihood framework. Both theoretically and experimentally, we demonstrate that SIG can generate high quality samples especially when dealing with multi-modality. By introducing SIG as an unbiased regularizer to the generative adversarial network (GAN), we show the interplay between maximum likelihood and adversarial learning can stabilize the adversarial training, resist the notorious mode collapsing problem of GANs, and improve the diversity of generated random samples. 

\end{abstract}

\section{Introduction}

Generative models consist of a group of fundamental machine learning algorithms that are used to estimate the underlying probability distributions over data manifolds. Promoted by recent development in deep neural networks, deep generative models achieve great success in data simulation, density estimation, missing data imputation, reinforcement learning and are widely utilized for tasks such as image super-resolution, compression and image-to-text translation. The goal of generative models is to minimize the distance between the generative distribution and data distribution under a certain metric or divergence $D$
\ba{
\min_{\phiv} D(P_{data}(\xv)||P_{model}(\xv;\phiv))
\label{eq:obj}
} 
where $P_{data}$ is usually approximated with empirical data distibution $\hat{P}_{data}=\frac{1}{N}\sum_{i=1}^{N}\delta_{\xv_i}$ based on observations $\{\xv_i\}_{1:N}$. 

Depending on the type of $P_{model}(\xv;\thetav)$, an existing generative model can often be classified as either an explicit generative model or implicit one. %
The former requires an explicit probability density function (PDF) for $P_{model}$ such that we can both sample data from it and evaluate its likelihood. Examples for explicit generative models include %
variational auto-encoders \cite{kingma2013auto}, PixelRNN\cite{oord2016pixel}, Real NVP\cite{dinh2016density}, and many Bayesian hierarchical models such as sigmoid belief net \cite{neal1992connectionist}.  An explicit generative model has a tractable density that can often be directly  optimized by \eqref{eq:obj}. The optimization target is a distance measure with nice geometric properties, which often leads  to stable training and theoretically guaranteed convergence. However, the requirement of having a tractable density usually restricts the flexibility of an explicit model, making it hard to scale with increasing data complexity. 

An implicit generative model, on the other hand, generates its random samples via a stochastic procedure but may not allow a point-wise evaluable PDF, which often makes a direct optimization in \eqref{eq:obj} become infeasible. Generative adversarial networks (GANs) \cite{goodfellow2016nips} %
 tackle this problem by introducing discriminator and solving a minimax game: a generative network generates random samples by propagating random noises through a deep neural network, whereas a discriminator aims to distinguish the generated samples from true data. 
Under the condition of having an optimal discriminator, training a vanilla GAN's generator is equivalent to optimizing \eqref{eq:obj} where $D$ is set as the Jensen-Shannon Divergence. Unfortunately in practice, the overall loss function of GAN is usually non-convex and practitioners have encountered a variety of obstacles such as gradient vanishing, mode collapsing, and high sensitivity to the network architecture \cite{arjovsky2017wasserstein,goodfellow2016nips,salimans2016improved,radford2015unsupervised}. 

To incorporate highly expressive generative model while maintaining a well-behaved optimization objective, we introduce semi-implicit generator (SIG), a Bayesian hierarchical generative model that mixes a specified distribution $P(\xv\given\thetav)$  with an implicit distribution $P_{\phiv}(\thetav)$
where the implicit distribution can be constructed by deterministically transforming random noise $ \zv_i $ to $\thetav_i$ using a $\phiv$ parameterized deterministic transform as $\thetav_i=g_{\phiv}(\zv_i),~\zv_i \sim p(\zv)$ which shares the same spirit as the recently developed semi-implicit variational inference (SIVI) \cite{pmlr-v80-yin18b}  for flexible posterior approximation.
Intuitively, $P(\xv\given\thetav)$ can incorporate our prior knowledge on the observed data, such as the data support, while $P_{\phiv}(\thetav)$ can maintain the high expressiveness. With the hierarchical structure, SIG can be directly trained by choosing $D$ as the Kullback-Leibler(KL) divergence and estimating %
\eqref{eq:obj}  with Monte-Carlo estimation. We show the SIG optimization objective can intrinsically resist the mode-collapse problem. 
By leveraging adversarial training, we apply SIG as a semi-implicit regularizer to generative adversarial networks, which helps stabilize optimization, significantly mitigates mode collapsing, and generates high quality samples in natural image scenarios. %

\section{Semi-implicit generator}

Defining a family of parametric distribution $P_{model}(\xv;\thetav)$, a classic explicit generative model is trained by maximizing the log-likelihood as
\ba{
\max_{\thetav} \frac{1}{N}\sum_{i=1}^{N}\log P_{model}(\xv_i;\thetav),
\label{eq:mle}
}
which is identical to minimizing cross-entropy $\bH(\hat{P}_{\text{data}},P_{model})= -\bE_{\hat{P}_{data}(\xv)}\log P_{model}(\xv;\thetav)$.
Assuming $P_{\text{data}}=\hat{P}_{\text{data}}$, %
which is independent of the optimization parameter $\thetav$, minimizing this cross-entropy is equivalent to \eqref{eq:obj} where $D$ is set as the KL divergence.

Instead of treating $\thetav$ as a global optimization parameter, we consider $\thetav$ as local random variable generated from distribution $p_{\phiv}(\thetav)$ with parameter $\phiv$. Semi-implicit generator (SIG) is defined in a two-stage manner 
\ba{
  \xv_i\sim p(\xv\given\thetav_i),~\thetav_i\sim p_{\phiv}(\thetav)
} 
Marginalizing $\thetav_i$ out, we can view the generator as $P_{model}(\xv_i;\phiv)=\int p(\xv_i\given\thetav_i)p_{\phiv}(\thetav_i)d\thetav_i$. Here $p(\xv_i\given\thetav_i)$ is required to be explicit but $p_{\phiv}(\thetav_i)$ can be defined by sampling a random variable $\zv_i$ from fixed distribution $p(\zv)$ and setting $\thetav_i = g_{\phiv}(\zv_i)$, where $g_{\phiv}: \cZ \to \cX$ is a deterministic mapping represented by neural network with parameter $\phiv$. Therefore, typically $p_{\phiv}(\thetav)$ cannot be evaluated pointwise and the marginal $P_{model}(\xv;\phiv)$ is implicit. Notice in this setting, $\thetav$ is required to be continuous while $\xv$ can be sampled from discrete distribution with continuous parameters.  

Minimizing the cross-entropy as $\bH(P_{\text{data}},P_{model})$  is equivalent to minimizing the  KL-divergence with respect to the model parameter as in \eqref{eq:obj}
\ba{
&\min_{\phiv} \kl(P_{data}(\xv)||P_{model}(\xv;\phiv)) \label{eq:kl} \\
\Leftrightarrow &\min_{\phiv} \bH(P_{data},P_{model}) = -\bE_{P_{data}(\xv)}\log\left(\int p(\xv\given\thetav)p_{\phiv}(\thetav)d\thetav\right) \label{eq:cross}
}
We show below that SIG can be trained %
 by minimizing an upper bound of the cross entropy in \eqref{eq:cross}. %
\begin{lemma}
Let us construct an estimator for the cross-entropy $\bH(P_{\text{data}},P_{model})$  as 
\ba{\bH_M = -\bE_{p_{\text{data}}(\xv)}\bE_{\thetav_1,\cdots,\thetav_M \sim p_{\phiv}(\thetav)}\log \frac{1}{M}\sum_{j=1}^M p(\xv\given\thetav_j),\label{eq:HM}}
 then for all $M$, $\bH(P_{\text{data}},P_{model}) \leq \bH_{M+1} \leq \bH_M$ and $\bH(P_{\text{data}},P_{model})  = \lim_{M \to \infty} \bH_M$. When $M=1$, let $\thetav^* = \argmin_{\thetav}-\bE_{P_d(\xv)}\log [p(\xv|\thetav)]$ then $\bH_1  \geq -\bE_{P_{data}(\xv)}\log[p(\xv|\thetav^*)] $ where the equality is true if and only if $p_{\phiv}(\thetav) = \delta_{\thetav^*}(\thetav)$.
\label{lemma:cross}
\end{lemma}

In practice, $\bH_M$ is approximated with Monte-Carlo samples as $-\frac{1}{N}  \sum_{i=1}^{N}\log \frac{1}{M}  \sum_{j=1}^{M} p(\xv_i \given \thetav_j)$, where $\{\xv_i\}_{1:N}$ and $\{\thetav_j\}_{1:M}$ are two sets of Monte Carlo samples generated from $P_{\text{data}}(\xv)$ and implicit $P_{\phiv}(\thetav)$, respectively. %
When $M=1$, the local $\thetav_i$ will degenerate to the same $\thetav^*$ and the objective degenerate to \eqref{eq:mle}.
To analyze the performance of SIG,  we first consider multi-modal data on which popular deep generative models such as GANs often fail due to mode collapsing. For theoretical analysis, we first define a discrete multi-modal space as follows.
\begin{definition}
(Discrete multi-modal space) Suppose $(\cX,\nu)$ is a metric space with metric $\nu : \cX \times \cX \mapsto \bR^+$,  $\cX = \bigcup\limits_{i=1}^{K}U_i$ , where $U_i \cap U_j = \emptyset $ for $i \neq j$. Let the distance between two sets be $D(U_i,U_j) = \inf\{\nu(x,y) ; x\in U_i, y \in U_j\}$ and let the diameter of a set be $d(U) = sup\{\nu(x,y) ; x,y\in U\}$. Suppose there exists $c_0 > \epsilon_0 > 0$ such that $\min_{i,j} D(U_i,U_j)  > c_0$, $\max_i d(U_i)  < \epsilon_0$. Then $\cX = \bigcup\limits_{k=1}^{K}U_k$ is a discrete multi-modal space under mesure $\nu$.
\end{definition}

Strictly speaking, there could be sub-modes within each $U_i$, but the above definition emphasizes the existence of multiple separated regions in the support. Since the loss of a deep neural network is a nonconvex problem, finding the global optimality condition for $\phiv$ can be difficult \cite{shang1996global,yun2017global}.  Thanks to the structure of SIG as a two-stage model, assuming the implicit distribution is flexible enough, we can study a simplified optimal assignment problem: assuming that $N$ data points have been sampled from the true data distribution, how to assign $M$ generated data to the neighborhood of the true data such that $\bH_M$ defined in \eqref{eq:HM} is minimized under expectation
\ba{
\min_{\{m_1,\cdots, m_k\}} -\frac{1}{N}  \sum_{i=1}^{N}\log \frac{1}{M}  \sum_{j=1}^{M} \bE_{\xv_i\in U_{t_i}, \thetav_j \in U_{z_j}} [p(\xv_i \given \thetav_j)],
\label{eq:sig2}
}
where the data are assumed to be generated from a discrete multi-modal space $\cX = \bigcup\limits_{i=1}^{K}U_i$, $\xv_i\in U_{t_i}$, $\thetav_j \in U_{z_j}$, $t_i,z_j\in\{1,\ldots,K\}$ and $\{m_k\}_{k=1}^K$ are the number of $\thetav$'s that are assigned to be in $U_k$. Assuming the data distribution is the marginal distribution of a normal-implicit mixture as $\xv_i\sim \cN(\xv_i\given\thetav_i, \sigma^2 I),~\thetav_i\sim p_{\phiv}(\thetav)$ and $U_i$ are equally spaced, we have the following theorem.
\begin{theorem}
(SIG for multi-modal space)
Suppose $P_{data}$ is defined on a discrete multi-modal space $\cX = \bigcup\limits_{i=1}^{K}U_i$ with $l_2$-norm. Suppose there are $N$ data points $\xv_i \sim P_{data}, i=1,\cdots,N$, among which $n_k$ points belong to $U_k$. Suppose we need to sample $\thetav_j\sim p_{\phiv}(\thetav), j=1,\cdots,M$, and $m_k$ denotes the number of $\thetav$'s in $U_k$.  Denoting $r$ as a radial basis function (RBF), we let $u = \bE [r(\xv, \thetav)]$ if $\xv,\thetav \in U_i$, and $v = \bE [r(\xv, \thetav)]$ if $\xv \in U_i$, $\thetav \in U_j, i \neq j$.  Then the objective in \eqref{eq:sig2} is convex and the optimum $m_k$ to maximize \eqref{eq:sig2} satisfies $\frac{m^*_k}{M} = \frac{n_k}{N} + (\frac{n_k}{N} - \frac{1}{K})\frac{Kv}{(u-v)}$. In particular, $m^*_k \neq 0$ if $n_k > \frac{N}{K} \frac{1}{1+\frac{u-v}{Kv}}$.
\label{thm:mode}
\end{theorem}

\begin{cor}
Assume $\xv \sim \cN(\muv_x, \sigma^2I)$ and $\thetav  \sim \cN(\muv_{\thetav}, \sigma^2I)$. Let $u = \bE \exp(\frac{-\norm{\xv-\thetav}^2}{2\sigma^2})$ if $\muv_x = \muv_{\thetav}$ and $v = \bE \exp(\frac{-\norm{\xv-\thetav}^2}{2\sigma^2})$ if $\norm{\muv_x - \muv_{\thetav}} = c\sigma$, then  $\frac{v}{u-v} = \frac{1}{e^{c^2/6 -1}}$.
\label{cor:mode}
\end{cor}
The ideal proportion for $\frac{m^*_k}{M}$ would be $\frac{n_k}{N}$, and $(\frac{n_k}{N} - \frac{1}{K})\frac{Kv}{u-v}$ plays the role as bias. In the normal-implicit mixture case, as shown in Corollary \ref{cor:mode}, if $ \xv \in U_{t_i}$, $\thetav \in U_{z_j}$ are approximately normal distributed, $\frac{v}{u-v}$ can be exponentially small  for well separated modes. This indicates that SIG has a strong built-in  resistance to model collapsing. %

There is  an interesting connection between SIG and %
variational auto-encoder (VAE) \cite{kingma2013auto,rezende2014stochastic}. VAE tries to maximize the evidence lower bound as $\mbox{ELBO}=\E_{p_{\text{data}}(\xv) q(\zv\given \xv)} \log (q(\zv\given \xv) / p_{\text{model}}(\xv,\zv)))$which is the same as maximizing
\ba{ 
-\kl(q(\zv\given \xv)p_{\text{data}}(\xv)||p_{\text{model}}(\xv,\zv)),
\label{obj:vae}
}
for which the optimal solution is $
q(\zv\given \xv)p_{\text{data}}(\xv) = p_{\text{model}}(\xv,\zv) = p(\xv\given \zv) p(\zv).
$
Therefore, VAE imposes the constraint that there exits a recognition network/encoder $q(\zv\given \xv)$, which is inferred by minimizing the KL divergence from $p_{\text{model}}(\xv,\zv)=p(\xv\given \zv) p(\zv)$, the joint distribution of the model,  to $\hat{p}_{\text{data}}(\zv,\xv)=q(\zv\given \xv)p_{\text{data}}(\xv)$, the joint distribution specified by the data distribution and encoder.

In SIG, we maximize 
\ba{
-\kl(p_{\text{data}}(\xv)|| \textstyle\int p(\xv\given \zv) p(\zv) d\zv) = -\kl( \int q(\zv\given \xv)p_{\text{data}}(\xv) d\zv || \int p(\xv\given \zv) p(\zv)d\zv),
}
where $q(\zv\given \xv)$ can be any valid probability density/mass function. 
VAE tries to match the joint distribution between the data combined with its encoder and the model, whereas SIG only cares about matching the marginal model distribution and the data distribution.
It is clear that SIG does not require a specific encoder structure and hence provides more flexibility. 

In experiments, we find that SIG can generate high-quality data samples on relatively simple data manifolds such as MNIST, but observe that the richness of its generated images can be hard to scale well with high data complexity, such as CelebA dataset with 200K $409 \times 687$ RGB images. More specifically, when setting $M=100$, we find the effect of ``mode averaging'' on generated images for complex data.   %
We suspect that $M$ needs to scale with data complexity such that $\bH_M$ is close to $\bH(P_{\text{data}},P_{model})$ and this is the price we pay for SIG to have a stable training with a strong resistance against mode collapsing. 
While SIG performs well on relatively simple data but suffers from ``mode averaging''  on complex natural images, the generative adversarial network (GAN) %
has shown the ability to generate high quality samples with large scale observed data,  but suffers from ``model collapsing'' even on a simple mixture of Gaussians.   %
To benefits from both worlds, we apply SIG as a regularizer in adversarial learning,  which can produce realistic samples, while strongly resisting both the mode collapsing and unstable optimization problems that are notorious for the training of GANs.

\section{Generative adversarial network with semi-implicit regularizer}
Generative adversarial network (GAN)\cite{goodfellow2014generative} solves a minimax problem
\ba{
\min_G\max_{D} V(D,G) = \E_{\xv\sim p_{data}(\xv)} [ \log D(\xv)] + \E_{\zv\sim p(\zv)} [\log (1-D(G(\zv))]
\label{obj:gan}
}
It is shown in \cite{goodfellow2016nips,goodfellow2014distinguishability}  that if the generator loss is changed from $\E_{\zv} [\log (1-D(G(\zv))]$ to $\frac{1}{2}\bE_{\zv}\exp(\sigma^{-1}(D(G(\zv))))$, with ideally optimal discriminator, the generator loss \eqref{obj:gan} is identical to the SIG loss \eqref{eq:kl}, which means SIG can be considered as training with the GAN's objective, using %
the optimal discriminator %
in the update of the generator. 
The discriminator in GAN can be considered as an augmented part of the model to avoid density evaluation and indirectly feed the information of the real data to optimizing the generator. With the help of the discriminator, the weak fitting of generator to real data brings high expressive samples that go beyond memorizing inputs. However, recently extensive research in both practical experiments \cite{radford2015unsupervised,metz2016unrolled} and theoretical analysis \cite{li2017towards,zhang2017discrimination,arjovsky2017towards} show that the lack of capacity, insufficient training of the discriminator, and the mismatches between the generator and discriminator in both network types and structures are the root causes of a variety of obstacles in GAN training. It also has been observed in \cite{goodfellow2014generative} and highlighted in \cite{metz2016unrolled,arjovsky2017wasserstein} that the optimal generator for a fixed discriminator is a sum of delta functions at the $\xv$'s, where the discriminator assigns the highest value, which eventually collapses the generator to produce a small family of similar samples. In comparison, SIG is trained by maximizing likelihood without using a discriminator, which can be considered as a strong fitting between real data and generated samples directly. This encourages us to combine the two models and apply SIG as a regularizer in a GAN model,  which is referred to as GAN-SI.

For GAN-SI, the discriminative loss is
\ba{
\min_{\gammav} - \bE_{\xv\sim P_{data}} \log D_{\gammav}(\xv) - \bE_{\zv \sim p(\zv)} \log(1-D_{\gamma}(g_{\phiv}(\zv))) \label{sigan:loss_d}
}
and generator loss is a linear combination of the original GAN loss and SIG loss as
\ba{
\min_{\phiv} -\bE_{\zv \sim g(\zv)} [\log D_{\gamma}(g_{\phiv}(\zv)) - \lambda \bE_{\xv \sim P_{data}} \log \int p(\xv \given \thetav)p_{\phiv}(\thetav)d\thetav], \label{sigan:loss_g}
}
where $\gammav$ are the discriminator network parameters,  $\thetav = g_{\phiv}(\zv)$ is the deterministic transform for the implicit distribution in SIG. We choose $p(\xv\given \thetav)$ as $\cN(\xv;\thetav,\sigma^2I)$ for image generation, and set $\lambda\geq0$ as a hyperparameter to balance the strength between the GAN and SIG objectives. In practice, we set $\lambda$ such that the GAN's generator loss and the cross-entropy term in \eqref{sigan:loss_g} are on the same scale. The neural networks are set according to the DCGAN  \cite{radford2015unsupervised}.

Since SIG can be considered as training GAN with a theoretically optimal discriminator, by adjusting $\lambda$, we are able to interpolate between the standard GAN training and  true generator loss, therefore balancing the discrimination-generalization trade-off in the GAN dynamics \cite{zhang2017discrimination}. This idea is related to Unrolled GAN \cite{metz2016unrolled} in which the discriminator parameter is temporarily updated $K$ times before updating the generator and the look-forwarded discriminator parameters are used to train the current generator. By adjusting the unrolling steps $K$,  Unrolled GAN can also interpolate between the standard GAN $(K=0)$ and optimal discriminator GAN $(K=\infty)$. However in Unrolled GAN, the discriminator for $(K=\infty)$ is not the theoretically optimal discriminator but a fully optimized one that is still influenced by the network design and data complexity. The effectiveness of Unrolled GAN in improving stability and mode-coverage is explained by the intuition that the training for the generator with looking ahead technique can take into account the discriminator's reaction in the future, thus helping spread the probability mass. But there is no theoretical analysis provided yet.  Moreover, the interpolation is non-linear, a few orders of magnitude slower as shown by \cite{srivastava2017veegan}, which makes picking $K$ not easy. Training GAN with a semi-implicit regularizer benefits from  both theoretical explanation and low extra computation, and shows the improved performance on reducing mode collapsing and increasing the stability of optimization in multiple experiments. 

\section{Related work}
Using a two-stage model is related to Empirical Bayes (EB) \cite{robbins1956empirical,casella1985introduction}. A Bayesian hierarchical model can be represented as $\xv_i \sim p(\xv_i\given \thetav_i), \thetav_i \sim p(\thetav_i \given \phiv), \phiv \sim p(\phiv)$,  where $p(\phiv)$ is a hyper-prior distribution. In EB, the hyperprior $p(\phiv)$ is dropped and the data is used to provide information about $\phiv$ such that the marginal likelihood $\prod_i p(\xv_i\given \phiv)$ is maximized. %
Previous learning algorithms for EB are often based on simple methods such as Expectation-Maximization and moment-matching. SIG can be considered as a %
parametric EB model where  the neural network parameters are represented by $\phiv$  and the training objective is to find %
the maximum marginal likelihood (MMLE) solution of $\phiv$ \cite{carlin1997bayes}. 
 
Without an explicit probability density, the evaluation of GAN has been considered challenging. There have been several recent attempts to %
 introduce maximum likelihood to the GAN training \cite{che2017maximum,grover2018flow}.  Flow-GAN \cite{grover2018flow} constructs a generative model based on normalizing flow,  which has been proven as an effective way to expand the distribution family in variational inference. Normalizing flow, however, requires the deterministic transformation to be invertible,  a constraint that is often too strong to allow it to generate satisfactory random samples by its own. Therefore,  its main use is to be  combined with GAN to help improve its sample quality. %

There has been significant recent  interest in 
 improving the vanilla GAN objective. For example, the measure between the data and model distributions can be changed  to the KL-divergence \cite{goodfellow2014distinguishability} or Wasserstein distance \cite{arjovsky2017wasserstein}; variational divergence estimation and density ratio estimation approaches have been used to extent the measure to a family of $f$-divergence \cite{nowozin2016f,poole2016improved}; a mutual information term has been introduced  into the generator loss to enable learning disentangled representation and visual concepts \cite{chen2016infogan}; and based on a heuristic intuition, two regularizers with an auxiliary encoder are introduced to stabilize the training and improve mode-catching, respectively \cite{che2016mode}.

A variety of GAN research focuses on solving the mode collapse problem via new methodology and/or theoretical analysis. Encoder-decoder GAN architectures, such as MDGAN \cite{che2016mode}, VEEGAN \cite{srivastava2017veegan},  BiGAN\cite{donahue2016adversarial}, and ALI\cite{dumoulin2016adversarially},  use an encoding network to learn reversed mapping from the data to noise. The intuition is that training an encoder can force the system to learn meaningful mapping that can transform imbedded codes to data points from different modes.
Unrolled GAN \cite{metz2016unrolled}, as discussed in the previous section, interpolates between the vanilla GAN discriminator and optimal discriminator that resists mode collapsing. %
AdaGAN \cite{tolstikhin2017adagan} takes a boosting-like approach which is trained on weighted samples with more weights assigned to missing modes. From a theoretical perspective, it is shown that if the discriminator  size is bounded by $p$, even the generator loss is $\epsilon$ close to optimal, the output distribution can be supported only on $\mathcal{O}(p\log p/\epsilon^2)$ images \cite{arora2017generalization}.  A simplified GMM-GAN is used to theoretically show that the optimal discriminator dynamics can converge to the ground truth in total
variation distance, while a first order approximation of the discriminator leads to unstable GAN dynamics and mode collapsing \cite{li2017towards}.  A negative conclusion is made that the encoder-decoder training objective cannot learn meaningful latent codes and avoid mode collapsing \cite{arora2017theoretical}. These theoretical analyses do support our practice of combining the GAN and SIG objectives.

\section{Experiments}
In this section, we first demonstrate the stability and mode coverage property of SIG on synthetic datasets. The toy examples show SIG can capture skewness, multimodality, and generate both continuous and discrete random samples that are indistinguishable from the true data. By interpolating between MLE and adversarial training scheme, we show GAN-SI can balance sample quality and diversity on real dataset. The evaluation criterion of generative model, however, is not straight-forward and no single metric is conclusive on its own. Therefore, we exploit multiple metrics to cross validate each other and emphasize quality and diversity separately. We notice the GAN training is sensitive to network structure, hyper-parameters, random initialization, and mini-batch feeding. To make a fair comparison, %
we share the same network structure between different generative models in each specific experiment setting and do multiple random trials. The results support the theorem that SIG can stably cover multi-modes and training GAN-SI adversarially greatly mitigates mode collapsing in GANs.

\subsection{Toy examples}
We first show the expressiveness of SIG with both discrete and continuous true data. For the discrete data, SIG is set as $x \sim \po(\theta), ~\theta\sim p_{\phiv}(\theta)$ where $p_{\phiv}(\theta)$ is implicit distribution generated by mapping from ten dimensional random noise with a two-hidden-layer multi-layer perceptron (MLP). The top left figures correspond to $P_{data} = \mbox{NegativeBinomial}(r=2,p=0.5)=\int \mbox{Poisson}(\theta)\ga(\theta;2,1)d\theta$ and bottom left figures correspond to $P_{data} = \frac{1}{2}\mbox{Poisson}(10)+\frac{1}{2}\mbox{NegativeBinomial}(r=0.2,p=0.9)$. For continuous data, SIG is set as $\xv \sim \cN(\thetav,0.1^2I), \thetav\sim p_{\phiv}(\thetav)$, where the $p_{\phiv}(\thetav)$ is the same as that for the discrete cases.  As Figure~\ref{fig:toy} show,  the implicit distribution is able to recover the underlying mixing distribution such that the samples following the marginal distribution can well approximate the true data. Vanilla GAN, as comparison, can only generate samples whose similarity to true data is restricted by the discriminator and cannot recover the original data well. 
 
\begin{figure}[ht]
\centering
\includegraphics[width=0.9\textwidth]{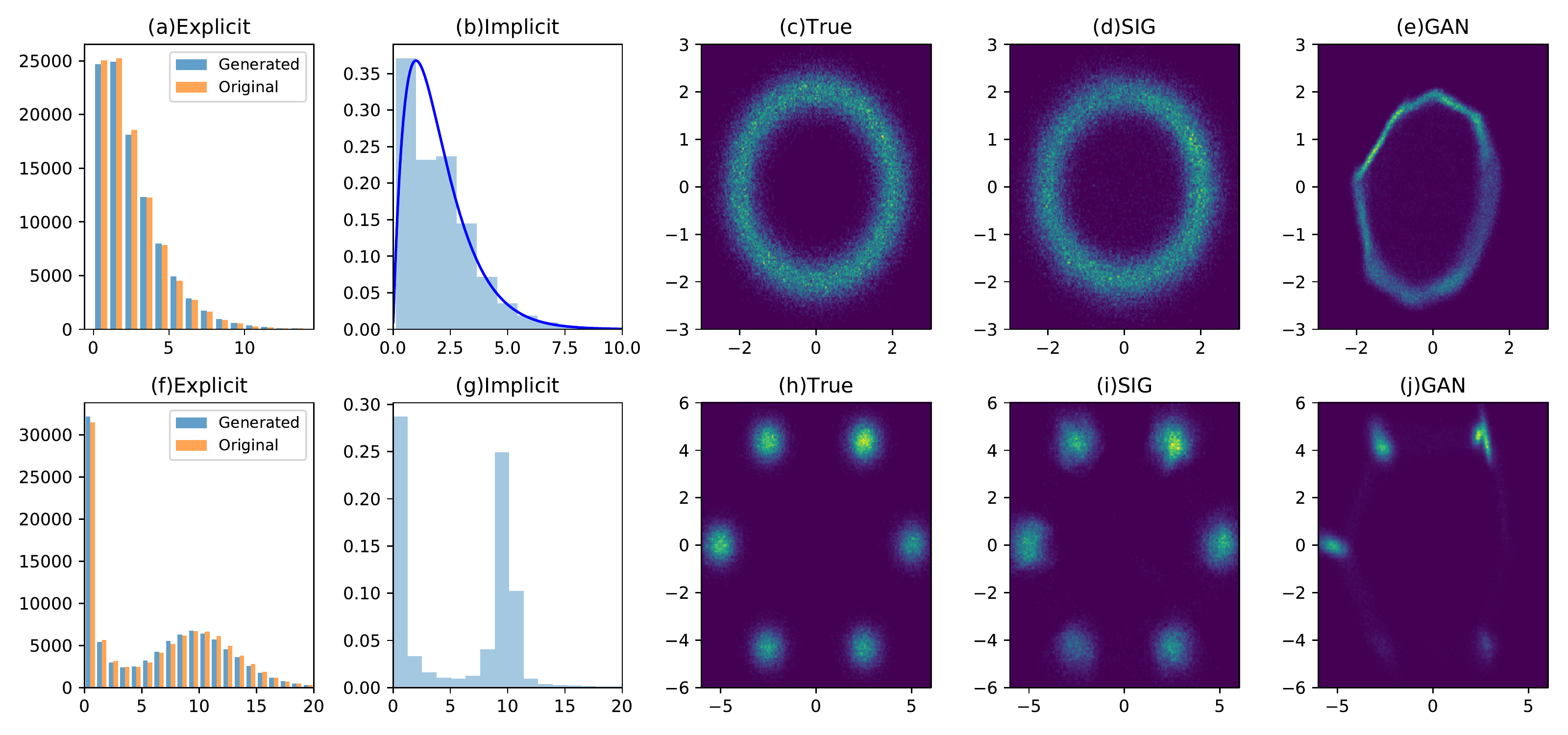}
\caption{Generated samples from SIG with true data coming from: (a)-(b): Negative-Binomial distribution; the implicit distribution can learn the true mixing distribution $\ga(\theta;2,1)$; (f)-(g): Mixture of Poisson and Negative-Binomial distribution; (c)-(e): Ring+Gaussian noise; (h)-(j): Gaussian mixture arranged on a ring.}
\label{fig:toy}
\end{figure}

\subsection{Mixture of Gaussians}
We compare different generative models on a $5\times5$ Gaussian mixture model. For fair comparison, all the models share the same generative network: a two-hidden-layer MLP with size 100 %
and rectified linear units (ReLU) activation function. The discriminator for GAN has a fully connected layer with size 100,  and the encoder for VAE and VEEGAN is a two-hidden-layer MLP with size 100.  

Detecting mode collapsing on a large dataset is challenging but can be accurately measured on synthetic data. To quantitatively evaluate the sample quality, we sample 50,000 points from trained generator and count it as high quality sample if it is within three standard deviations away from any of the mixture component centers. A center that is associated with %
more than 1000 high quality samples will be counted as a captured mode. The proportion of high quality samples at each mode, together with the proportion of low quality samples, form a 26 dimensional discrete distribution $P_g$. We calculated the KL divergence $\kl(P_g||P_{data})$. All results are reported based on the average and standard error of five independent random trials. %

\begin{figure}[ht]
\centering
\includegraphics[width=0.9\textwidth, height=6.5cm]{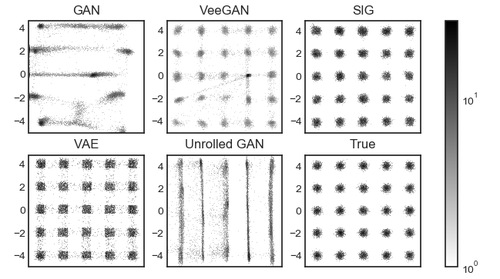}
\label{fig:gmm}
\caption{Comparison of generated sample for Gaussian mixture model by vanilla GAN, modified GAN to reduce mode collapsing(Unrolled GAN, VEEGAN), VAE and SIG.}
\end{figure}

\begin{table}[H]
\centering
\caption{Comparison of mode-capturing ability on mixture of Gaussian. 'Modes' is the number of captured modes out of 25. 'KL' is $\kl(P_{\text{model}}||P_{data})$. For 'Modes' and 'Proportion of high quality samples', the higher the better; for 'KL', the lower the better.}
\centerline{
\begin{tabular}{|c|c|c|c|c|c|}
\hline
 &GAN  & VAE & VEEGAN & Unrolled GAN&  SIG  \\ \hline 
Modes & 4.0$\pm$3.08   & 25$\pm$0 & 23.2$\pm$0.84 & 6.2$\pm$8.6 & \textbf{25$\pm$0} \\\hline
\makecell{Proportion of high\\ quality samples}	 & 0.36$\pm$0.16   & 0.83$\pm$0.02 & 0.82$\pm$0.03 & 0.42$\pm$0.13 & \textbf{0.91$\pm$0.04}  \\ \hline
KL & 2.87$\pm$0.78 & 0.32$\pm$0.07 & 0.38$\pm$0.08 & 1.97$\pm$0.60  & \textbf{0.14$\pm$0.07}  \\ \hline
\end{tabular}%
}
\label{tab:gmm}
\vspace{-3mm}
\end{table}

As shown in Table \ref{tab:gmm}, SIG captures all the modes and generates the highest proportion of high quality samples, whose %
distribution is closest to the ground truth. It also achieves the shortest running time and highest stability using a single neural network.

We notice, however, SIG generalization ability may not scale well with increasing data complexity, as shown in Figure \ref{fig:sig}. To generate natural images, we train SIG adversarially and notice the proposed GAN-SI can stabilize GAN training and mitigate the mode collapsing problem.

\begin{figure}[H]
\centering
\begin{tabular}{ccc}
\includegraphics[width=0.28\textwidth, height=3.7cm]{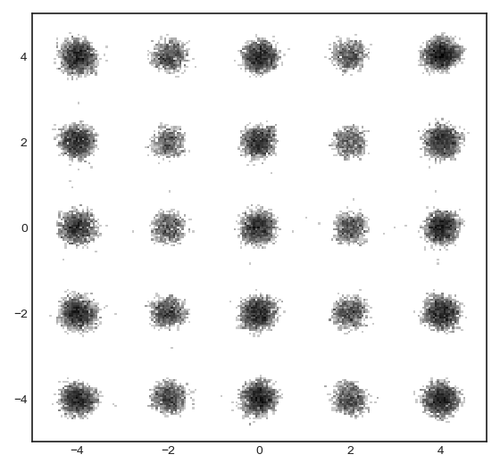}\hspace{-1em}&
\includegraphics[width=0.28\textwidth, height=3.7cm]{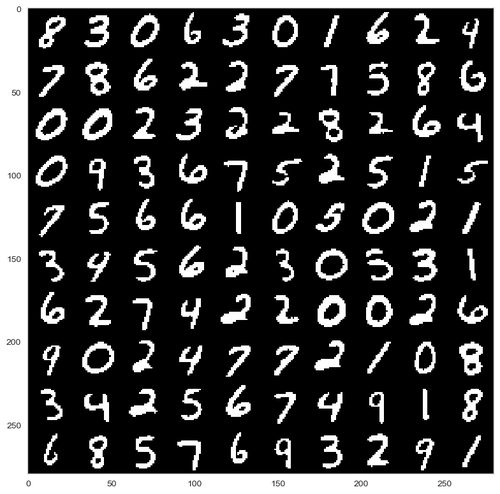}\hspace{-1em}&
\includegraphics[width=0.28\textwidth, height=3.7cm]{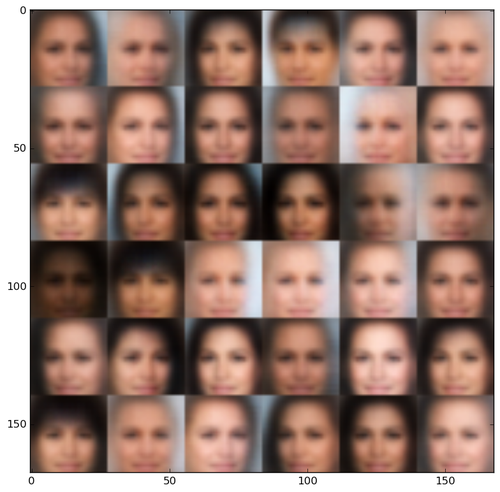}\hspace{-1em}\\
(a)&(b)& (c)
\end{tabular}
\caption{(a): SIG can generate low complexity data well. The input data is from unevenly distributed GMM, where the components in the 1st, 3rd, and 5th columns has twice more data than the 2nd and 4th. SIG generates samples well aligned with the true distribution. (b): SIG generated MNIST digits. (c): SIG scales not well when the data
is as complex as CelebA.}  %
\label{fig:sig}
\end{figure}

\subsection{Stacked MNIST}
To measure the performance of combining MLE and adversarial training schemes on discrete multimodal data, we stack 3 randomly chosen MNIST images on the RGB color channels to form a $28\times28\times3$ image (MNIST-3) \cite{srivastava2017veegan,metz2016unrolled,che2016mode,tolstikhin2017adagan}.  MNIST-3 contains 1000 modes corresponding to 3-digit between 0 and 999. Similar to \cite{metz2016unrolled} and \cite{tolstikhin2017adagan} we find the missing modes problem of GAN on MNIST3 is sensitive to the network architecture and the randomness in training process due to the instability. Therefore, we choose three different network sizes (denoted as S, M, and L), run each experiment five times and use exactly the same generator and discriminator for DCGAN and DCGAN-SI. 

The inception score (IS) \cite{salimans2016improved} is a widely used criterion for GAN evaluation. It is applied to data $\xv$ with label $y$ using a pre-trained classifier. Low entropy $H(y|x)$ of conditional distribution $p(y\given\xv)$ and high entropy $H(y)$ of marginal distribution $p(y)$ are considered to represent high image quality and diversity respectively. 
\ba{
\text{IS} = \exp(\bE[\kl(p(y\given\xv)||p(y))]) = \exp(H(y) - \bE[H(y|x)])
}

As the IS by itself cannot fully characterize generative model performance \cite{barratt2018note,borji2018pros}, we  provide more metrics for evaluation: High quality image means the proportion of images that can be classified by the trained classifier with a probability larger than 0.7; Mode is the number of digit triples that have at least one sample; KL is $\kl(P_g||P_{data})$ where $P_{data} = (\frac{1}{1000},\cdots,\frac{1}{1000})$. 

\begin{table}[H]
\footnotesize
\caption{High quality image and $\exp(H(y|x))$ reflect sample quality while $\exp(H(y))$, Mode and KL reflect sample diversity. For Inception score, High quality image, $\exp(H(y))$, higher is better; for $\exp(H(y|x))$ and KL, lower is better.}\vspace{-1.5mm}
\centering
\centerline{
\begin{tabular}{ccccccc}
\toprule
	 & IS & High quality  & $\exp(H(y|x))$ & $\exp(H(y))$ & Mode & KL \\ \midrule
	DCGAN(S) & 2.9$\pm$0.52 & \textbf{0.63$\pm$0.14} & \textbf{1.96$\pm$0.32} & 5.1$\pm$1.19 & 21.0$\pm$8.12 & 4.99$\pm$0.24 \\ %
	DCGAN-SI(S) & \textbf{4.33$\pm$0.59} & 0.6$\pm$0.07 & 2.05$\pm$0.2 & \textbf{8.78$\pm$0.41} & \textbf{279.2$\pm$296.52} & \textbf{2.63$\pm$1.0} \\ \midrule
	DCGAN(M) & 5.59$\pm$0.36 & 0.7$\pm$0.03 & 1.71$\pm$0.09 & 9.51$\pm$0.31 & 811.8$\pm$116.24 & 0.75$\pm$0.35 \\ %
	DCGAN-SI(M) & \textbf{5.93$\pm$0.47} & \textbf{0.72$\pm$0.04} & \textbf{1.65$\pm$0.11} & \textbf{9.75$\pm$0.11} & \textbf{969.0$\pm$29.19} & \textbf{0.3$\pm$0.13} \\ \midrule
	DCGAN(L) & 4.71$\pm$1.12 & 0.67$\pm$0.08 & 1.78$\pm$0.17 & 8.25$\pm$1.32 & 389.8$\pm$477.24 & 2.95$\pm$2.33 \\
		DCGAN-SI(L) & \textbf{6.05$\pm$0.68} & \textbf{0.73$\pm$0.06} & \textbf{1.62$\pm$0.17} & \textbf{9.75$\pm$0.12} & \textbf{957.0$\pm$32.74} & \textbf{0.36$\pm$0.12} \\ \bottomrule
\end{tabular}%
}
\vspace{-3mm}
\end{table}

\subsection{Sample quality and diversity on CIFAR10}
We test the semi-implicit regularizer on the CIFAR-10 dataset, a widely studied dataset consisting of 50,000 training images with $32\times32$ pixels from ten categories. The image diversity is high between or within each category. We combine semi-implicit regularizer with two popular GAN frameworks DCGAN  \cite{radford2015unsupervised} and WGAN-GP\cite{gulrajani2017improved} to balance the quality and diversity of generated samples. 

\begin{table}[h]
\small
\centering
\caption{Inception scores for models on CIFAR-10}
\resizebox{.8\textwidth}{!}{
\begin{tabular}{c|cccc}
\toprule 
\multirow{2}{*}{Real data} &  \multicolumn{4}{c}{Unsupervised, standard CNN} \\  \cline{2-5} 
& DCGAN & DCGAN-SI & WGAN-GP & WGAN-GP-SI \\\midrule
\multicolumn{1}{c}{11.24 $\pm$ .12} &6.16 $\pm$ .14 & 6.85$\pm$ .06  &  6.43 $\pm$ .07 & 6.67 $\pm$ .11\\
\bottomrule
\end{tabular}
}
\label{tab:cifar}
\end{table}

We train each model for 100K iterations with mini-batch size 64. The optimizer is Adam with learning rate 0.0002. The inception model we use is based on pre-trained Inception Model\cite{szegedy2016rethinking} on ImageNet. As shown in Appendix Figure \ref{fig:cifar}, the images generated by DCGAN include duplicated images indicating the existence of mode collapsing which does not seem to happen with regularized DCGAN-SI, and this is reflected in the improvement of inception score as shown in Table \ref{tab:cifar}.

\section{Conclusions}
We propose semi-implicit generator (SIG) as a flexible and stable generative model. Training under well-understood maximum likelihood framework, SIG is proposed either as a black-box generative model or as unbiased regularizer in adversarial learning. We analyze the inherent mode-capturing mechanism and show its advantage over several state-of-the-art generative methods in reducing mode collapse. Combined with GAN,  semi-implicit regularizer provides an interplay between adversarial learning and maximum likelihood inference, leading to a better balance between sample quality and diversity.

\clearpage
 \bibliographystyle{unsrt}
\bibliography{reference.bib,References052016.bib}
\clearpage

\newpage
\appendix
\begin{center}\Large{\bf{

 Semi-implicit generative model:
 supplementary material}}
\end{center}
\section{Proofs}
\begin{proof}[Proof of Lemma \ref{lemma:cross}]
Assume integer $K>M>0$ and $\cI$ is the set of all size $M$ subsets of $\{1,\cdots,K\}$. {Let $I$ be a discrete uniform random variable that takes outcome $\{i_1,\cdots, i_M\}$ in $\cI$ with probability $P(I = \{i_1,\cdots, i_M\}) = \frac{1}{\binom KM}$.} We have $\bE_I \frac{1}{M} \sum_{j \in I}  p(\xv|\thetav_j) = \frac{1}{K} \sum_{j=1}^K p(\xv|\thetav_j)$. Then we have
\bas{
\bH_K = &-\bE_{\xv \sim P_{data}(\xv)}\bE_{\thetav_1,\cdots,\thetav_K \sim p_{\phiv}(\thetav)}\log \frac{1}{K}\sum_{j=1}^K p(\xv|\thetav_j)\\
=&-\bE_{P_{data}}\bE_{\thetav_1,\cdots,\thetav_K \sim p_{\phiv} (\thetav)} \log \bE_{I} \frac{1}{M}\sum_{j \in I}  p(\xv|\thetav_j)  \\
\leq & -\bE_{P_{data}}\bE_{\thetav_1,\cdots,\thetav_K \sim p_{\phiv} (\thetav)}  \bE_{I} \log\frac{1}{M}\sum_{j \in I}  p(\xv|\thetav_j) \\
=& -\bE_{P_{data}}\bE_{\thetav_1,\cdots,\thetav_M \sim p_{\phiv} (\thetav)}  \log\frac{1}{M}\sum_{j =1}^M  p(\xv|\thetav_j) = \bH_M
}
By law of large number $\log\frac{1}{M}\sum_{j =1}^M  p(\xv|\thetav_j)$ converges to $\log \int p(\xv|\thetav)p(\thetav)d\thetav$ a.s.  as $M\to\infty$ so $\bH_M \to \bH({P_{data}},{P_{\text{model}}})$. When $M=1$, assume $\thetav^* = \argmin_{\thetav}-\bE_{P_{data}(\xv)}\log [p(\xv|\thetav)]$ then 
\bas{
-\bE_{P_{data}(\xv)}\log [p(\xv|\thetav)] \geq -\bE_{P_{data}(\xv)}\log [p(\xv|\thetav^*)]
}
Multiply both sides by $p_{\phiv}(\thetav)$ and integrate over $\thetav$, we have
\bas{
-\bE_{\thetav\sim p_{\phiv}(\thetav)}\bE_{P_{data}(\xv)}\log [p(\xv|\thetav)] \geq -\bE_{P_{data}(\xv)}\log [p(\xv|\thetav^*)]
}
The minimal is reached when implicit distribution degenerates to the point probability mass $p_{\phiv}(\thetav) = \delta_{\thetav^*}(\thetav)$ where $\theta^*$ maximizes average log-likelihood over data. 
\end{proof}

\begin{proof}[Proof of Theorem \ref{thm:mode}]
Suppose $P_{data}$ is defined on a discrete multi-modal space $\cX = \bigcup\limits_{i=1}^{K}U_i$. For $\xv_i \sim P_{data}, i=1,\cdots,N$, assume $\xv_i \in U_{t_i}$; for $\thetav_j \sim P_{\phiv}(\thetav), j=1,\cdots,M$, assume $\thetav_j \in U_{z_j}$, where $t_i$, $z_j$ denote the mode label of true data and generated data center respectively. Then we have $n_k = \sum_{i=1}^N \mathbbm{1}\{t_i = k\}$ and $m_k = \sum_{j=1}^M \mathbbm{1}\{z_j = k\}$ for $k=1,\cdots,K$. 
\ba{
\min_{\{m_1,\cdots, m_k\}} -\frac{1}{N}  \sum_{i=1}^{N}\log \frac{1}{M}  \sum_{j=1}^{M} \bE [p(\xv_i \given \thetav_j)]
\Leftrightarrow  \min_{\{m_1,\cdots, m_k\}} -  \sum_{i=1}^{N}\log  \sum_{j=1}^{M} \bE \exp^{-\frac{\norm{\xv_i-\thetav_j}^2}{2\sigma^2}}
\label{eq:sig3}
}
Notice $\sum_{k=1}^{K}n_k = N$, $\sum_{k=1}^{K}m_k = M$. By definition of $\cX$, if $\xv,\thetav \in U_k$, $r(\xv,\thetav) = \exp^{-\frac{\norm{\xv-\thetav}^2}{2\sigma^2}} \geq \exp^{-\frac{\epsilon_0^2}{2\sigma^2}}$ and if $\xv \in U_i, \thetav \in U_j~ i \neq  j$, $r(\xv,\thetav) = \exp^{-\frac{\norm{\xv-\thetav}^2}{2\sigma^2}} \leq \exp^{-\frac{c_0^2}{2\sigma^2}}$. With the definition of $u$ and $v$ in theorem \ref{thm:mode}, we have $u \geq \exp{-\frac{\epsilon_0^2}{2\sigma^2}} > -\frac{c_0^2}{2\sigma^2} \geq v$. Then we have objective (\ref{eq:sig3}) as a constrained optimization problem with Lagrange multiplier $\beta$
\bas{
\min_{\{m_1,\cdots, m_k\}} & - \sum_{i=1}^{N}\log  \sum_{j=1}^{M} \bE r(\xv_i,\thetav_j) + \beta(\sum_{k=1}^K m_k - M)\\
= & - \sum_{i=1}^{N}\log (m_{t_i}u + (M - m_{t_i})v) + \beta(\sum_{k=1}^K m_k - M) \\
= & - \sum_{k=1}^{K} n_k \log (m_k + Mv - m_kv) + \beta(\sum_{k=1}^K m_k - M) 
}
Taking the gradient with respect to $(m_1,\cdots, m_k)$ and set to zero gives
\bas{
&\frac{\partial}{\partial m_k} - \sum_{k=1}^{K} n_k \log (m_k + Mv - m_kv) + \beta(\sum_{k=1}^K m_k - M) \\
= & \frac{-n_k(u-v)}{m_k(u-v)+Mv} + \beta = 0,~~~ \text{for} ~~k=1, \cdots, K
}
Together with constraint $\sum_{k=1}^K m_k = M$, we have
\ba{
\frac{m^*_k}{M} = \frac{n_k}{N} + (\frac{n_k}{N} - \frac{1}{K})\frac{Kv}{(u-v)}
}
The Hessian $H = \text{diag}(\frac{n_k(u-v)^2}{(m_k(u-v)+M)^2})\succ 0$ shows convexity and $\frac{m^*_k}{M}$ is global minimum. Let the right-hand-side greater than 0, the condition for mode k not vanishing is $n_k > \frac{N}{K} \frac{1}{1+\frac{u-v}{Kv}}$
\end{proof}

\begin{proof}[Proof of Corollary \ref{cor:mode}]
Assume $\xv\sim \cN(\muv_x, \sigma^2I)$, $\thetav\sim \cN(\muv_{\theta}, \sigma^2I)$. Let $\zv = \frac{\xv - \thetav}{\sqrt{2}\sigma}$ and $\muv = \muv_x - \muv_\theta$ then $\zv \sim \cN(\frac{\muv}{\sqrt{2}\sigma}, I)$. Let $\chi = \zv^T\zv$, then $\chi$ follows noncentral chi-squared distribution $\chi \sim \chi(N, \lambda)$ where $N$ is the dimension of $\zv$, $\lambda=\frac{\muv^T\muv}{2\sigma^2}$ is noncentrality parameter. By moment genrating function (MGF) of noncentral chi-squared distribution, we have
\ba{
&\bE e^{-\frac{\norm{\xv-\thetav}^2}{2\sigma^2}} = \bE e^{-\zv^T\zv} \notag \\ 
=& \bE e^{-\chi} = MGF_{\chi}(-1) \notag \\
=& 3^{-N/2} e^{-\lambda/3}  \label{eq:mgf}
}
For $u$, $\muv=0$, $\lambda=0$ and for $v$, $\norm{\muv} = c\sigma$, $\lambda=\frac{c^2}{2}$. Plugging into \eqref{eq:mgf}, we have $u=3^{-N/2}$, $v=3^{-N/2}e^{-c^2/6}$,  therefore $\frac{v}{u-v} = \frac{1}{e^{c^2/6}-1}$.
\end{proof}

\section{Algorithm for GAN-SI}
\begin{algorithm}[ht] 
\While{not converged}{
Sample $\xv_i \sim P_{data}(\xv)$ for $i\in\{1,\cdots,N\}$ \; 
Sample $\zv_j \sim g(\zv)$ for $j \in \{1,\cdots,M\}$ \; 
Set $\thetav_j = g_{\phiv}(\zv_j)$ for all $j$ \; 

$g_{\gammav} \leftarrow -\nabla_{\gammav} \big[ \frac{1}{N} \sum_{i=1}^N\log D_{\gammav}(\xv_i) + \frac{1}{M} \sum_{j=1}^M \log(1-D_{\gammav}(\thetav_j))\big]$ \; 
$g_{\phiv} \leftarrow -\nabla_{\phiv} \big[ \frac{1}{M} \sum_{j=1}^M \log D_{\gammav}(\thetav_j) - \lambda \frac{1}{N} \sum_{i=1}^N \log \frac{1}{M} \sum_{j=1}^M p(\xv_i \given \thetav_j)\big] $\;
$\gammav \leftarrow \gammav- \eta g_{\gammav}$, $~~~\phiv \leftarrow \phiv- \eta g_{\phiv}$
}
The first order optimization is used as Adam\cite{kingma2014adam} in our experiments.
\caption{Mini-batch training of GAN-SI}\label{alg:sigan}
\end{algorithm}

\section{Network architecture and samples for MNIST-3}
The generator network is defined with parameter $K_g$ to adjust network size
\begin{table}[H]
\centering
\begin{tabular}{|c|c|c|c|c|}
\hline
 & Number of output & Kernel size & Stride& Padding \\  \hline
 Input $z\sim N(0,I_{256})$  & \multicolumn{4}{c|}{-} \\ \hline
 Fully connected & 4*4*64 & \multicolumn{3}{c|}{-} \\ \hline
 Transpose Convolution & 64*$K_g$ & 4& 1 & VALID \\ \hline
  Transpose Convolution & 32*$K_g$ & 4& 2 & SAME \\ \hline
    Transpose Convolution & 8*$K_g$ & 4& 1 & SAME \\ \hline 
    Convolution & 3 & 4& 2 & SAME \\ \hline
 \end{tabular}
\end{table}

The discriminator network is defined with parameter $K_d$ to adjust network size
\begin{table}[H]
\centering
\begin{tabular}{|c|c|c|c|c|}
\hline
 & Number of output & Kernel size & Stride& Padding \\  \hline
 Input is image batch with size 28*28*3  & \multicolumn{4}{|c|}{-} \\ \hline
 Transpose Convolution & 8*$K_d$ & 4& 2 & VALID \\ \hline
  Transpose Convolution & 16*$K_d$ & 4& 2 & SAME \\ \hline
    Transpose Convolution & 32*$K_d$ & 4& 1 & SAME \\ \hline 
    Flat+Fully connected& 1 & \multicolumn{3}{c|}{-}  \\ \hline
 \end{tabular}
\end{table}

For the network work size denoted as (S), (M), (L), the ($K_g$,$K_d$)  pair is chosen as $(1,0.5)$, $(1,1)$, $(2,1)$ respectively.

\section{Additional figures}
\begin{figure}[ht]
\centering
\includegraphics[scale=0.6,height=6.5cm]{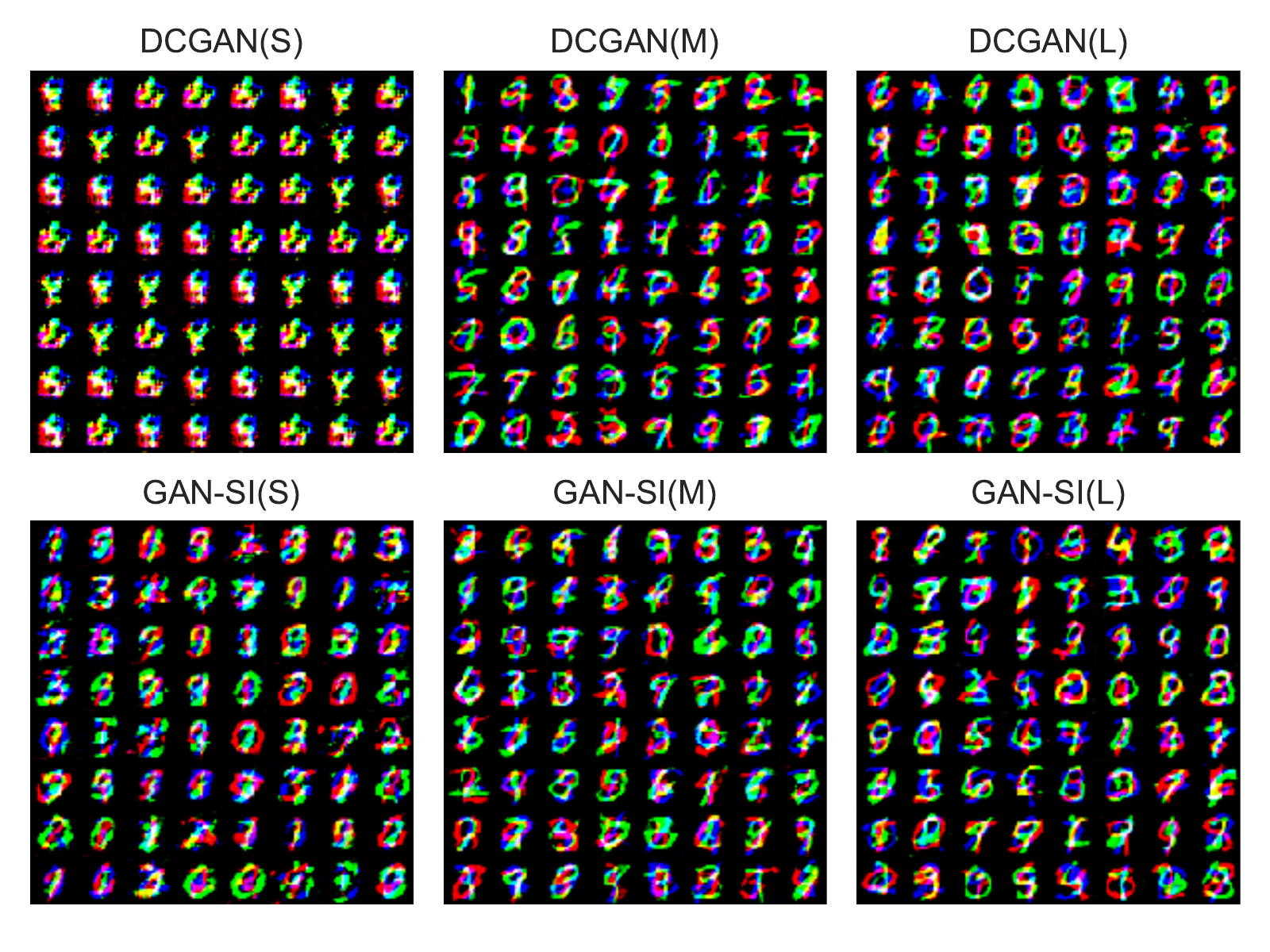}
\label{fig:mnist_good}
\caption{MNIST-3, highest inception score cases among 10 independent trials}
\end{figure}
\begin{figure}[ht]
\centering
\includegraphics[scale=0.6,height=6.5cm]{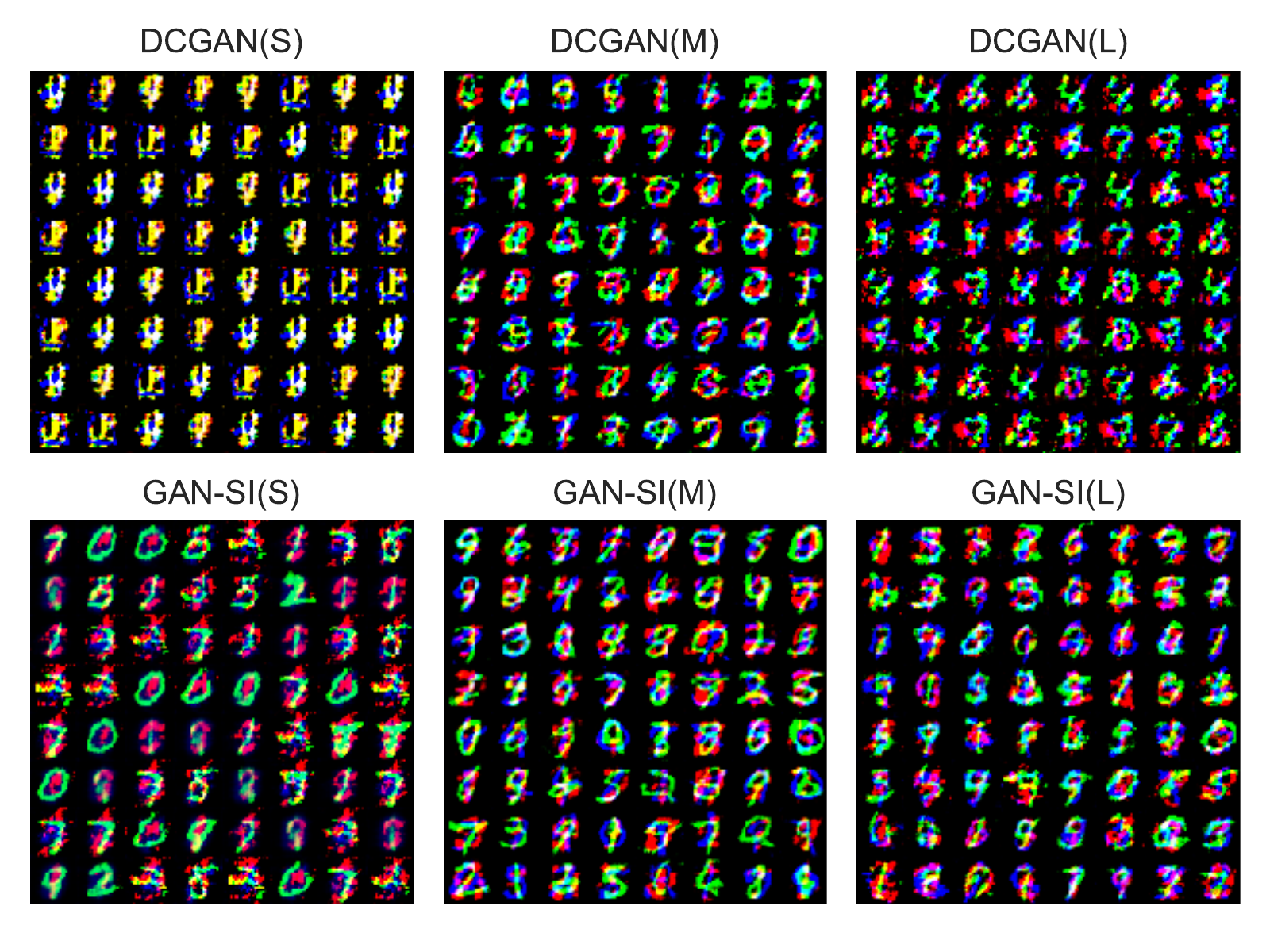}
\label{fig:mnist_bad}
\caption{MNIST-3, lowest inception score cases among 10 independent trials}
\end{figure}

\begin{figure}[t]
\centering
\begin{tabular}{c}
 \includegraphics[width=0.9\textwidth]{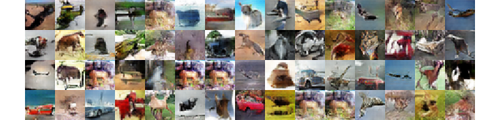}\\
\scriptsize (a)DCGAN \\ 
\includegraphics[width=0.9\textwidth]{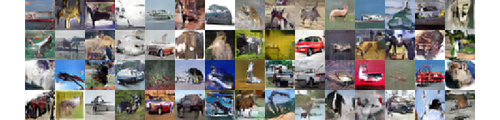}\\
\scriptsize (b)DCGAN-SI \\ 
\end{tabular}
\caption{ Randomly generated images by DCGAN and DCGAN with semi-implicit regularizer.}
\label{fig:cifar}
\end{figure}
\end{document}